\documentclass{article}
\usepackage{jfrExamplee}
\usepackage{graphicx}
\DeclareGraphicsExtensions{.pdf,.jpg,.png}
\usepackage{apalike}
\usepackage{setspace}
\usepackage[ruled]{algorithm}
\usepackage[noend]{algorithmic}
\usepackage{appendix}

\pdfoutput=1

\usepackage[ruled]{algorithm}
\usepackage{algorithmic}
\usepackage{amsmath,amsfonts,amssymb,array}
\usepackage{graphicx}        
\usepackage{epsfig}
\usepackage{subfigure}

\usepackage{xspace}
\usepackage{setspace}
\usepackage{verbatim}
\usepackage{amsthm}

\newcommand{\sref}[1]{Section~(\ref{#1})}
\newcommand{\aref}[1]{Algorithm~\ref{#1}}
\newcommand{\eref}[1]{Equation (\ref{#1})}

\newcommand{\cref}[1]{~\cite{#1}}

\newtheorem{thm}{Theorem}
\newtheorem{Lemma}{Lemma}
\newtheorem{Definition}{Definition}

\newcommand{\cFree}{\textrm{$C_{free}$}\xspace}
\newcommand{\ctrlSpace}{\textrm{$U$}\xspace}

\newcommand{\initConf}{\textrm{$q_0$}\xspace}

\newcommand{\stSpace}{\textrm{$S$}\xspace}

\newcommand{\dyn}{\textrm{$f$}\xspace}
\newcommand{\path}[1]{\textrm{$\gamma_{#1}(t)$}\xspace}

\newcommand{\ctsTraj}{\textrm{$\phi(t)$}\xspace}
\newcommand{\optCtsTraj}{\textrm{$\phi^*(t)$}\xspace}

\newcommand{\traj}{\textrm{$u(t)$}\xspace}

\newcommand{\ctsTrajp}{\textrm{$\phi'(t)$}\xspace}

\newcommand{\optCtrlSeq}{\textrm{${u}^*$}\xspace}

\newcommand{\ctrlSeq}{\textrm{$u$}\xspace}
\newcommand{\ctrl}{\textrm{$u$}}
\newcommand{\ctrlp}{\textrm{$u'$}}

\newcommand{\errs}{\textrm{$E_S$}\xspace}
\newcommand{\erru}{\textrm{$E_U$}\xspace}
\newcommand{\derrs}{\textrm{$\dot{E}_S$}\xspace}

\newcounter{comment}

\title{Analysis of Asymptotically Optimal Sampling-based Motion Planning Algorithms for  Lipschitz Continuous Dynamical Systems}

\author{
Georgios Papadopoulos \\
Department of Mechanical Engineering\\
Massachusetts Institute of Technology\\
77 Massachusetts Avenue, Cambridge, MA, USA   \\
\texttt{gpapado@mit.edu} \\
\And
Hanna Kurniawati \\
School of Inf. Technology \& Electrical Engineering \\
University of Queensland \\
St Lucia, Brisbane, QLD, Australia \\
\texttt{hannakur@uq.edu.au} \\
\And
Nicholas M. Patrikalakis \\
Department of Mechanical Engineering\\
Massachusetts Institute of Technology\\
77 Massachusetts Avenue, Cambridge, MA, USA   \\
\texttt{nmp@mit.edu} \\ 
}

\begin{document}

\maketitle


\begin{abstract}
Over the last 20 years significant effort has been dedicated to the development of sampling-based motion planning algorithms such as the Rapidly-exploring Random Trees (RRT) and its asymptotically optimal version (e.g. RRT$^*$). However,  asymptotic optimality for RRT$^*$ only holds for   linear  and fully actuated systems or for a small number of non-linear systems (e.g. Dubin's car) for which a steering function is available. The purpose of this paper  is to show that asymptotically optimal motion planning for dynamical systems with differential constraints  can be achieved without the use of a steering function.    We develop a novel analysis on sampling-based planning algorithms that sample the control space. This analysis demonstrated that asymptotically optimal path planning for any Lipschitz continuous dynamical system can be achieved by sampling the control space directly. We also determine theoretical bounds on the convergence rates for this class of  algorithms.  As the number of iterations  increases, the trajectory generated by these algorithms, approaches the optimal control trajectory, with probability one. Simulation results are promising.

\end{abstract}

\section{Introduction}
In this paper, we are interested in optimal motion planning for robots with challenging dynamics. Given a robot with perfectly known dynamics and an environment map that includes the initial state of the agent, the goal region, and the obstacles, an optimal motion planner computes a set of control inputs that drive the agent from the initial state to the goal region with minimum cost and without colliding with the obstacles.

In this paper, we are interested in robots with high-dimensional dynamics, thus we are using a sampling-based approach \cite{kavraki2005Book}, \cite{lavalle2006planning}, which is the most successful approach for high dimensional motion planning problems. Up to a few years ago, sampling-based motion planners were only known to be probabilistically complete, in the sense that given enough time, they will find an admissible trajectory with probability one, whenever such a trajectory exists. The question of optimality was addressed recently in \cref{Karaman.Frazzoli:IJRR11}. They developed sampling-based algorithms that are asymptotically optimal, such as PRM$^*$ and RRT$^*$, where asymptotically optimal means given enough time, the probability the planner returns a trajectory close to the optimal trajectory, is one. 

Although PRM$^*$ and RRT$^*$ perform well for robots with simple dynamics, such as holonomic robots, extending asymptotically optimal sampling-based planners to systems with complex dynamics remains an open problem. The difficulty lies in the fact that both RRT$^*$ and PRM$^*$ perform their sampling in the configuration space and therefore requires a steering function that drives the robot from one configuration to the other. Unfortunately, steering functions are often difficult to find and may not even exist for systems with complex dynamics, such as systems with non-holonomic constraints.  The authors believe  that for the general case, where we have non-linear under-actuated robots, sampling the control space directly may be a better  choice. By sampling the control space directly, the need for a steering function can be eliminated, and the difficulty of dealing with challenging dynamics can be alleviated significantly.

The main contributions of this paper are as follows:
\begin{itemize}
\item We showed that asymptotic optimality in sampling-based motion planning, for any Lipschitz continuous dynamical system, can be achieved by sampling the control space directly. This result eliminates the need for a steering function, which is often computationally expensive or even infeasible to compute for systems with complex dynamics. 
\item  We derived a theoretical bound on the convergence rate for such planners. This result elucidate the problem parameters that effect the convergence rate (e.g. the dimension of the control space and the Lipschitz continuity constant of the underling dynamical system).
\end{itemize}
To the best of our knowledge, the aforementioned results are the first for sampling-based planners that sample the control space directly.  Our  analysis only requires Lipschitz continuity of the differential constraints and the cost function. Furthermore, we present a family of  asymptotically optimal  motion planners for systems with differential constraints that sample the control space directly, and hence does not require  a steering function.  This algorithms build upon existing sampling based planners\cref{Hsu97pathplanning}, and compute admissible trajectories with decreasing cost. 
 
This paper is organized as follows. In the next section, Section (2), we describe related work,  Section (3) formally defines the motion planning problem. In Section (4) we describe  a family of algorithms that sample the control space and discuss their complexity and in the Analysis section (Section(5)), which is the main part of this paper, we provide our theoretical analysis that  includes a novel optimality analysis and a study on the convergence rates of algorithms that sample the control space.  In Section (6) we give some simulation results and finally we close with conclusions and future work in Section (7).

\section{Related Work}
Over the past two decades, several sampling-based motion planners have been proposed. Sampling-based planners can be divided into two broad categories, i.e., graph based approaches (multi-query) and tree based approaches (single-query). Graph-based approaches first construct a random graph in the configuration space and then use the graph to find  admissible paths. This approach includes Probabilistic RoadMap (PRM)\cite{kavraki+1996:prm} \cite{kavraki+1996:ana}, and it variants  \cite{Sun05:Narrow},\cite{Kur06:On},\cite{Ekenna2013iros},\cite{Hsu05:On}, \cite{AmaBayDalJonVal98}. A good summary of sampling-based approaches can be found in \cite{kavraki2005Book}.

Tree-based approaches construct a random tree and stop construction whenever the constructed tree contains a path from a given initial configuration to the goal region. This approach includes the widely used Rapidly-exploring Random Trees algorithm  (RRT) \cite{Lavalle98rapidly-exploringrandom} and its variants e.g. \cite{Yershova2007tro}. Another type of tree-based sampling based planner is the Expansive Space Tree (EST)\cite{Hsu97pathplanning}.  In terms of handling differential constraints, the main difference between the RRT planner  and the  EST planner is that the former performs its sampling in the configuration space and the latter performs its sampling in the control space.

Recently, Karaman and Frazzoli showed that RRT and PRM converge to a non-optimal trajectory and they  developed asymptotical optimal planners e.g. RRT$^*$ and PRM$^*$ \cite{Karaman.Frazzoli:IJRR11},\cite{Karaman2010kinodynamic}. However, a steering function is required to guarantee optimality and thus they are not suitable for general systems with differential constraints. Perez et al. proposed the use of the RRT$^*$  algorithm combined with LQR methods to solve path planning problems with challenging dynamics \cite{Perez2012icra}. Similarly, Tedrake et al. proposed the LQR-Trees for feedback motion planning \cite{Tedrake10lqr-trees:feedback}. However, these methods perform well only close to the linearization area. As a result  the optimality properties of the RRT$^*$ do not hold for the LQR-RRT based algorithms.  More recently, Dobson and Bekris \cite{95kostas} proposed an algorithm that relaxes asymptotic optimality to near-optimality, to speed-up computing the solution. However, they still require a steering function.

To avoid the use of steering function and linearization, one can sample the control space directly using forward propagation. Such planners were initially proposed by Hsu et al. \cite{Hsu97pathplanning}. However, \cite{Hsu97pathplanning} only finds an admissible trajectory, rather than the optimal one.
Recently, Littlefield et al.\cite{Littlefield2013IROS} proposed a planner that combines RRT and direct sampling of the control space, called Sparse-RRT. It builds a tree by choosing a node to expand using the RRT strategy, i.e., sampling a configuration uniformly at random and expanding a node that is nearest to the uniformly sampled configuration. However, it expands the node using random shooting, based on forward propagation of the dynamics.  They showed that Sparse-RRT converges to the near-optimal solution. However, the proof is based on an assumption that random sampling in the control space, i.e., the random shooting process, will converge to a near optimal solution. In this paper, we show that planners that directly sample the control space can converge to the optimal solution without using the aforementioned assumption. 

Our proof is based on the simplest algorithm in the class of planners that sample the control space directly, i.e., uses uniform random sampling to choose the node to expand, and expands the node by sampling the control space uniformly at random. The idea is by showing that this simplest algorithm converges to the optimal trajectory, we have more confident that there will be a more efficient sampling strategies that samples the control space directly and converges to the optimal solution. And therefore, spending more effort to find a better strategy to directly sample the control space would be a worth while pursuit. 

We also present comparison results between the simplest strategy used for proving the convergence result, and the more sophisticated sampling strategy, including one that chooses a node to expand using the RRT strategy (similar to Sparse-RRT). Interestingly, our simulation results indicate that choosing a node to expand using the RRT strategy does not perform well when the system has challenging dynamics. The results and discussion on this are in \sref{s:results}. 

In summary, the main contribution of this paper is to analyse the sampling-based motion planning algorithms and shed some light on its challenges. We show that  asymptotically optimal planning can be achieved by directly sampling the control space. Furthermore, we present comparison results of different sampling-based planners that sample the control space directly. As the number of iterations increases, these algorithms generate a trajectory that approaches the optimal control trajectory.

\section{Problem Definition}
\label{ProblemDefinition}
 
 Suppose the set of all possible robot states is $S$ and the set of all possible control inputs is $U$. Then the equation of motion is given by: 
 \begin{equation}
	\dot{\gamma}(t) = \dyn(\gamma(t), \ctsTraj,t)
	\label{e:dyn}
\end{equation}
Where for each time $t$, $\gamma(t)$  $\in$ $\stSpace$ is a state,  $\dot{\gamma}(t)$ is the time derivative of the state and and $\ctsTraj \in \ctrlSpace$ is a control input. Assuming that $S$ and $U$ are manifolds of dimensions $n$ and $m$ where $m \leq n$, using  appropriate charts we can treat $S$ as a subset of $\mathbb{R}^n$ and $U$ as a subset of $\mathbb{R}^m$. Furthermore we assume that $f$ is a Lipschitz-continuous function with respect to state and control with maximum  Lipschitz constant to be less than $L_p \in \mathbb{R^+} $.

Supposed $S_{\text{free}} \subset S $ is the collision free subset of the state space. Then we define path to be a time parametrized mapping from time to the obstacle free state space e.g. $\gamma: [0,T] \rightarrow S_{\text{free}}$. We also define trajectory  or control function to be a time parametrized mapping from time to the control space e.g.  $\phi: [0,T] \rightarrow U$. From  Lipschitz-continuity it can be shown that for each trajectory (control function) there is a unique path that represents the solution to equation of motion \cite{slotine1991applied}.

Let $S_{\text{goal}} \subset S$ be the goal region. We would like to compute the control function (trajectory) and the resulting path that drives the agent from the initial state $\gamma(t=0)$ to the goal region and minimizes a given cost function.  More formally, let $D(\phi)$ be the resulting cost for control function   $\phi$ and $\Phi$ to be the set all admissible trajectories. An admissible trajectory is a trajectory that induces an obstacle free state-space path (through Equation (\ref{e:dyn})) and stops within the goal region. Among all admissible trajectories we would like to compute the one that minimizes the cost function, e.g
\begin{equation}
  \phi^* = \arg\min_{\phi \in \Phi} D(\phi)
\end{equation}

\section{Planners that Sample the Control Space}

In this section we describe the proposed family of algorithms. Since we are interested in systems with challenging dynamics,  all of the proposed algorithms sample the control space directly. Such algorithms are not new\cite{Hsu97pathplanning}\cite{Littlefield2013IROS}, but the question of whether such planners can converge to the optimal solution is still largely open. To answer this question, we first present two minor modifications of existing planners as detailed in \sref{s:algDesc}.

\subsection{Algorithm Description}
\label{s:algDesc}
\vspace{-0.5cm}

Each of the proposed algorithms constructs a tree  $\mathbb{T}=\{ \mathbb{M}, \mathbb{E} \}$, where the root is the initial state \initConf. 
Each node $q \in \mathbb{M}$ in the tree corresponds to a collision-free state while each edge $\overline{qq'} \in \mathbb{E}$ corresponds to a control input $u \in U$ that drives the robot from $q$ to $q'$ in a given time, without colliding with any of the obstacles while satisfying the equation of motion, Equation (\ref{e:dyn}).  Each node $q \in \mathbb{M}$ is annotated with the cost of reaching $q$ from \initConf. The algorithms are presented in \aref{a:main} -- \aref{a:expandRRT}.

To construct the tree, a node is chosen for expansion and then a control input is sampled uniformly at random. We present two different methods for selecting which node to expand. The first method chooses the node for expansion uniformly at random (\aref{a:expandUniform}), the second method chooses  the node for expansion in an $RRT$ style, \aref{a:expandRRT}, (with a Voronoi bias, similar to Sparse-RRT\cite{Littlefield2013IROS}). 

Once a node to be expanded is selected, the planner samples a control input uniformly at random, choose a time-interval $\Delta t$ on how long the sampled control input will be applied to the system, and use forward propagation to compute the resulting state. As we will see in \sref{s:analysis}, the integration time $\Delta t$ has to be chosen wisely. In order to  guarantee asymptotic optimality, the integration time has to either be sampled uniformly from zero to a maximum integration time or to decrease slowly and approach to zero as the number of iterations approaches to infinity. When constant integration time is used the integration time has to be sufficiently small; In the results section we present results for both cases (constant integration time and uniformly sampled integration time).

To improve computational efficiency, we can prune nodes that have cost larger than the cost of other nodes in their neighborhood (see \aref{a:prune}).  Every time we add a node, we search its neighborhood to find other nodes within a range $R(t)$. $R(t)$ monotonically shrinks as the number of iterations increases and starts from a given number e.g. $R_0$.  If there is a node within a ball of radius $R(t)$ with cost less than the cost of the new node, then we do not add the new node otherwise we add it. On the other hand,  given that we add the new node, we remove all the other nodes that have cost greater  than the cost of the new node.  We also make sure not to delete nodes that are part of the best trajectory found so far.

There are many different variants of the above algorithms that may substantially improve computational time. However, in this paper, we focus on answering the open question of whether a planner that samples the control space directly can converge to the optimal solution. To this end, we will analyze the simplest version of this class of planners, i.e., the uniform variant (\aref{a:expandUniform}). The idea is if we can show that asymptotic convergence holds even for this simplest algorithm, then it is more likely there will be a more efficient strategy to sample the control space such that asymptotic optimality holds, and therefore spending more effort finding a better strategy to directly sample the control space would be a worth while pursuit.

In addition to the analysis, we also present performance comparison between the three different sampling strategies mentioned above, on various motion planning problems involving robots with complex dynamics.

\vspace{0.1cm}
 \begin{algorithm}[!ht]
	\caption{Planner(MethodTobeUsed)}
	\label{a:main}
	\begin{algorithmic}[1]
	{\footnotesize
		\STATE Initialize: Set Initial configuration, Set Configuration Space,  Obstacle Data Structure, BestCost=$+\infty$, $R_o$
		\WHILE{(PlanningTime=TRUE)}
		         \STATE{R(t)=Shrink(Ro,t)}	
			\IF {Simple-Uniform}
				\STATE [Child, Parent, TrajFromParent, u]=ExpandTreeUniform($\mathbb{T}$);
			\ENDIF
			\IF {ExpandTreeRRT}
				\STATE [Child, Parent, TrajFromParent, u]=ExpandTreeRRT$^\prime$($\mathbb{T}$);	
			\ENDIF		
				\IF{(CollisionFree(TrajFromParent) $\wedge$ PossiblyOptimal(Child)}
				     \IF{(Prune(T,Child,R(t))=True)}
				          \STATE  NodeCost=Cost(Parent)+Cost(TrajFromParent)
						\IF{(NodeCost$<$BestCost)}
							\STATE NodeList=NodeList  $\cup$ Child 
							\STATE EdgeList=EdgeList $\cup$ $\overline{\textrm{ParentChild}}$
								\STATE BestCost=NodeCost
				    				\STATE BestNode=Child
				      \ENDIF	
				    		\ENDIF
				\ENDIF
		\ENDWHILE
	}
	\end{algorithmic}
\end{algorithm}
\vspace{-0.3cm}
\begin{algorithm}[!ht]
	\caption{ExpandTreeUniform($\mathbb{T}$)}
	\label{a:expandUniform}
	\begin{algorithmic}[1]
	{\footnotesize
		\STATE Sample a node: $ q' \leftarrow$ $rand()$
		\STATE Sample control input: $ u\prime \leftarrow$ $rand()$
		\STATE Propagate: $[q_{new}, TraFromParent ]  \leftarrow \int^{\Delta t}_0 {g}(q' ,u \prime) dt$
		\RETURN $q_{new}, q\prime, TraFromParent, u\prime$
			}
	\end{algorithmic}
\end{algorithm}
\vspace{-0.3cm}
\begin{algorithm}[!ht]
	\caption{ExpandTreeRRT$^\prime$($\mathbb{T}$)}
	\label{a:expandRRT}
	\begin{algorithmic}[1]
	{\footnotesize
	          \STATE Sample a node at random: $ q' \in \cFree$
	          \STATE Find the nearest: $q_{\text{near} }=FindNearest(T)$		
		\STATE Sample control input: $ u\prime \leftarrow$ $rand()$
		\STATE Propagate: $[q_{new}, TraFromParent ]  \leftarrow \int^{\Delta t}_0 {g}( q_{\text{near} } ,u \prime) dt$
		\RETURN $q_{new},  q_{\text{near} }, TraFromParent, u\prime$
			}
	\end{algorithmic}
\end{algorithm}
 \vspace{-0.3cm}
\begin{algorithm}[!ht]
	\caption{Prune$(\mathbb{T},Child,R(t)$)}
	\label{a:prune}
	\begin{algorithmic}[1]
	{\footnotesize
	          \STATE{[Neighbors]=NeighborsWithinRange(Child,R(t))}
	          \FOR {($i=0; i<$sizeof(Neighbors);$i++$)}
	            \STATE{CurrentNeighbor=Neighbors(i)}
	          \IF{CurrentNeighbor$\rightarrow$cost $< $Child$\rightarrow$cost}
	                   \STATE{ToBeRemoved(CurrentNeighbor)}
	                     \RETURN {False}
	          \ENDIF
	          \ENDFOR
                    \STATE{Remove(ToBeRemoved)}	        
	          \RETURN {True}
			}
	\end{algorithmic}
\end{algorithm}

\subsection{Complexity}
The pruning is the most expensive part of the proposed framework. To compute the neighbors within a range, the complexity is  $O(log(N))$ (where $N$ is the number of nodes in the tree). To do the cost comparison of each one of  the neighbors, the complexity is  $O(N_e)$ ($N_e$ is the number of neighbors, because of the pruning part $N_e<<N$). To remove nodes from the tree, the complexity is $O(N))$ if we choose to free the allocated memory or $O(N_e)$ if we do not.   Therefore the complexity of the proposed algorithm is $O(log(N))$ per iteration.

\section{Analysis}
\label{s:analysis}

In this section we will provide  a novel analysis on the optimality of planners that directly sample the control space. As we show here, asymptotically optimal planning for systems with differential constraints (e.g. non-holonomic systems) can be achieved without the use of a steering function. We start by showing that under certain conditions,   as the number of samples goes to infinity, the  probability a planner, that directly samples the control space,  samples a sequence of controls that are ``near" to the optimal control function goes to one. Afterwards, we show that under certain conditions, two nearby sequence of controls will induce state space paths that have similar costs.
The optimality proof  described in this section is proven for the case in which we choose a node to expand uniformly at random and without doing pruning.  We believe  that similar properties hold for the other (RRT). In the case we have non-linear systems, as we see from the results section, the convergence for an RRT-style algorithm is slower than the uniform one. Given we know the properties of the underlining sampling strategy (e.g. choose a node to expand and pruning characteristic), immediate results from the analysis below could allow researchers to evaluate any sampling strategy. 

\subsection{Optimality}

First, let's define the notion of ``nearby" sequence of control inputs more formally as
\begin{Definition}
\label{d:1}
Let $\ctrl = (u_0, u_1, \ldots, u_k)$ and $\ctrlp = (u'_0, u'_1, \ldots, u'_l)$ be two sequences of control inputs. Then, \ctrl\ and \ctrlp\ are $\epsilon$-close whenever $\max_{i \in [0, \max(k, l)]} \| u_i - u'_i \| \leq \epsilon$, where $u_i = 0$ for $i > k$ and $u'_i = 0$ for $i > l$.
\end{Definition}

Let $\optCtsTraj: [0,T_g] \rightarrow \ctrlSpace$ denote the optimal trajectory, where $T_g$ is the time to reach the goal state. Suppose the optimal trajectory is approximated with a step function, which is represented  as a sequence of control inputs $\optCtrlSeq = (u^*_1, u^*_2, \ldots, u^*_n)$, where $n = \lfloor \frac{T_g}{\Delta t} \rfloor$, $u^*_i = \phi^*((i-1) \cdot \Delta t)$ for $i \in [1, n]$, and each control input in the sequence is applied for $\Delta t$ time. We can then state our assumptions as,
\begin{itemize}
	\item The optimal trajectory \optCtsTraj is differentiable twice.
	\item Equation (\ref{e:dyn}) (the equation of motion) is Lipschitz continuous on both state and control arguments.
	\item The cost function is Lipschitz continuous.
\end{itemize}

 We also consider the standard $\delta$-clearance assumption \cite{kavraki2005Book}. This assumption is not essential for the convergence of the family of planners that sample the control space to the optimal trajectory (e.g. $\delta$ is arbitrarily small). If $\delta$ is large then we will be able to get admissible trajectories faster than the case $\delta$ is small, however from our analysis does not follow that the convergence to the optimal trajectory depends on $\delta$. 	

Before going into the details of our analysis we would like to outline our proof. Let $\phi^* : [0,T_g] \rightarrow U$ be the 
optimal trajectory and $u^* : [0,T_g] \rightarrow U$ to be a piecewise approximation of the optimal trajectory.\footnote{for this analysis  we are using piecewise constant functions (e.g. step functions)  with constant time intervals $\Delta t$; similar analysis can be derived for different functions and different choices of $\Delta t$.} In addition, we consider the trajectory $u : [0,T_g] \rightarrow U$ to be a trajectory returned by a planner that samples directly the control space. It is important to say that $u$ is of the same form as $u^*$ (e.g. both of them are piecewise constant functions), in addition we assume that $u$ is $\epsilon$-close to $u^*$. The above trajectories are illustrated in the Figure (\ref{fig:trajectories}).
\begin{figure}[htb]
  \begin{center}
\includegraphics[width=5 in]{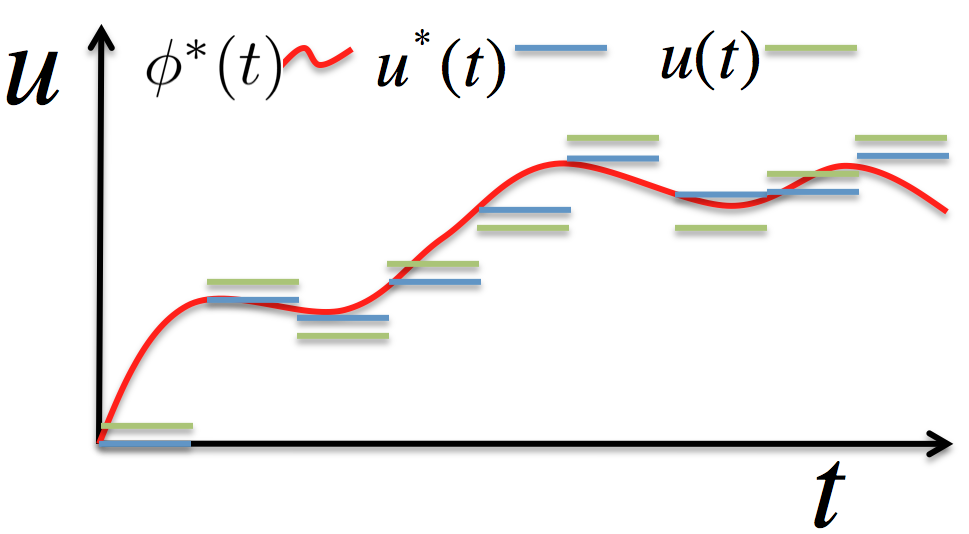}
      \caption{We can see the optimal control input ($\Phi^*(t)$), a representation of the optimal control input ($u^*(t)$)  and a control input that is close to the representation of the optimal control input ($u(t)$ e.g. a control an algorithm that samples the control space returns).}
    \label{fig:trajectories}
  \end{center}
\end{figure} 
We would like to prove that by sampling directly the control space the probability to get a trajectory that is sufficiently close to any approximation of the optimal trajectory approaches to 1 as the number of samples (in the control space) increases.

Before proving the asymptotic optimality property  for the proposed family of planners, we first need to relate the $\delta$-clearance property (which is defined in the state space) to properties on the control space such as the $\epsilon$-close property and the discretization interval $\Delta t$. To this end, we relate the distance between two trajectories (e.g. the optimal one and the one the proposed algorithm returns) with the distance between their induced state space paths. Any sampling based algorithm chooses discrete  control inputs (trajectory), however the optimal control input (trajectory) is  continuous. In our analysis we would like to take into account both distance due to approximation errors and due to the $\epsilon$-close property.
\begin{Lemma}
Suppose $\ctsTraj: [0, T_g] \rightarrow \ctrlSpace$ and $\ctsTrajp: [0, T_g] \rightarrow \ctrlSpace$ are two trajectories  that start at time $0$ and end at the same time $T_g$. Let \path{\phi} and \path{\phi'} be the state space paths induced by each trajectory, based on the equation of motion, \eref{e:dyn}. If at each  $t \in [0, T_g]$, \dyn is Lipschitz continuous in both state and control arguments, then $\| \path{\phi} - \path{\phi'} \| \leq \erru(e^{L_pt} -1) \leq  \erru (e^{L_pT_g}-1) $,  $\forall t \in [0, T_g]$, where, $L_p$ is the maximum Lipschitz constant and  $\erru = \max_{t \in [0,T]} \| \ctsTrajp - \ctsTraj \|$
\label{l:pathErr}
\end{Lemma}
\begin{proof}
We consider 2 control trajectories with the same time duration applied to the equations of motion with the same initial conditions, and we would like to study the distance of the resulting paths in the state space at each time.  Direct application of Lipschitz continuity on the equations of motion gives:
\small
\begin{eqnarray}
||\dot{\gamma}_{\phi}(t)-\dot{\gamma}_{\phi'}(t)|| = ||f(\gamma_{\phi}(t),\phi(t),t) -f(\gamma_{\phi'}(t), \phi'(t),t )|| \nonumber \\
 \leq L_p || \gamma_{\phi}(t) - \gamma_{\phi'}(t) || + L_p ||\phi(t)-\phi'(t)||
\label{eq:Lipschitz}
\end{eqnarray}
\normalsize
In the above equation $||.||$ indicates the $L_1$ Norm. Because the equation of motion is Lipschitz continuous, it is important to note that for every trajectory there is a unique path at each time \cite{slotine1991applied}.
 
Let $\erru = \max_{t \in [0,T]} \| \ctsTrajp - \ctsTraj \|$ and let $\errs(t) = \| \path{\phi} - \path{\phi'}  \|$. Then, for the $L_1$ norm we have $\derrs(t) \leq ||\dot{\gamma}_{\phi}(t)-\dot{\gamma}_{\phi'}(t)||$.  Direct substitution to Equation (\ref{eq:Lipschitz}) yields the differential  inequality below that describes the evolution of the error on the state space:
\begin{eqnarray}
\derrs(t)  \leq  L_p \errs(t) + L_p \erru
\end{eqnarray}
Solving the above differential equation and applying the boundary condition we get:
 \begin{eqnarray}
 \errs(t) \leq \erru (e^{L_pt}-1)  \leq \erru (e^{L_pT_g}-1), \forall t \in[0,T_g] 
 \label{eq:boundsU}
 \end{eqnarray}

\end{proof}

Given the maximum distance of 2 trajectories in the control space ($\epsilon$-close trajectories), the above Lemma shows a bound on the distance of the induced paths on the state space. 

\begin{Lemma}
\label{l:disErr}
Let $\ctsTraj: [0, T_g] \rightarrow \ctrlSpace$ be a continuous-time trajectory with maximum absolute value of the slope plus the reminder (the Peano form of the remainder) in all dimensions of  $\ctsTraj$  equal to $\alpha.  $\footnote{we do not consider pathological cases where the derivative can be infinity at few singular points e.g.  the tangent  function at zero. From the Taylor theorem using  the Peano form of the remainder we have that $f(t)= f(0) + (f'(t0) + h_1(t))\Delta t$, $h_1(t)$ goes to zero faster than $\Delta t$} Suppose $[0, T_g]$ is discretized into uniform interval $\Delta t$ and $\traj: [0, T_g] \rightarrow \ctrlSpace$ is a step function approximation of \ctsTraj where $\traj = \phi(\lfloor\frac{t}{\Delta t}\rfloor \Delta t)$. If \path{\phi} is the state-space path induced by \ctsTraj and \path{u} is the state-space path induced by \traj, then $\| \path{\phi} - \path{u}\| \leq m \alpha \Delta t (e^{L_pt}-1)$.
\end{Lemma}
\begin{proof}
Using Taylor expansion series for  $\ctsTraj$ (in  all dimensions of the control space) we get:
\begin{eqnarray}
|| \ctsTraj - \traj|| \leq  \sum_1^m \alpha \Delta t = m \alpha \Delta t , \forall t \in [0,T_g]
 \end{eqnarray}
 Using results from Lemma (\ref{l:pathErr}) we get:
\begin{eqnarray}
\| \path{\phi} - \path{u}\| \leq m \alpha \Delta t (e^{L_pt}-1), \forall t \in [0,T_g]
 \end{eqnarray} 

\end{proof}

Applying the above 2 Lemmata to the optimal trajectory,  to its representation (approximation) and to a trajectory that is $\epsilon$-close to the representation of the optimal trajectory, we can define the appropriate $\Delta t$ for a given $\delta$ and $\epsilon$ to be:
\begin{equation}
\|\gamma_{\phi^*} - \gamma_{u} \| \leq (m \alpha \Delta t +\epsilon)(e^{L_p T_g} -1) \leq \delta
\label{eq:deltaBound}
\end{equation}
Using the above requirement for $\Delta t$, we can prove the following convergence theorem.

\begin{thm}
Let $\optCtrlSeq$ denote the optimal sequence of control inputs when the time domain is discretized uniformly into $\Delta t$ intervals. Then, for any $\epsilon > 0$ and $\Delta t$ that satisfies Equation (\ref{eq:deltaBound}), as the number of samples goes to infinity, the probability the proposed planer  samples at least one sequence of control input $u$ that has the same number of elements as \optCtrlSeq  and is $\epsilon$-close to \optCtrlSeq, goes to one.
\label{th:prob}
\end{thm}
\begin{proof}
Each path from the root to a node of the tree $\mathbb{T}$  encodes a particular sequence of control input ( e.g $\overline{qq'}$, the edges of the tree are labeled with control input). The probability the proposed algorithm selects a particular sequence to be extended is the same as the probability it selects a node of  $\mathbb{T}$  to be expanded, which is  $\frac{1}{j}$.

Let $\rho_i$ be the probability the proposed algorithm samples a control input within $\epsilon$ distance from $u^*_i$. Because the new planner samples control inputs uniformly  at random (from the control space) then, for all $i \in [0, |u^*|]$, $\rho_i = \rho = \frac{Vol(\mathbb{B_\epsilon}) }{Vol(U)} $. Where $Vol(.)$ is the volume function.

Let $P_{j, k}$ be the probability that from $j$ valid samples (e.g. for up to $j$ iterations), the proposed algorithm generates at least one control sequence \ctrlSeq that is $\epsilon$-close to the first $k^{th}$ subsequence of \optCtrlSeq, i.e, $|\ctrlSeq| = k$ and $dist(u_i, u^*_i) \leq \epsilon$ for $i \in [1, k]$.  As we show in the Appendix (A), this probability can be written as: 
\begin{eqnarray}
P_{j,k}&\geq& P_{j-1,k} + \frac{\rho}{j} ({1}- \frac{P_{j-1,k}}{P_{j-1,k-1}}) P_{j-1,k-1} \hspace{0.2cm}\text{or} \\
P_{j,k}&\geq& P_{j-1,k} + \frac{\rho}{j} (P_{j-1,k-1}- P_{j-1,k})
\label{eq:conv1}
\end{eqnarray}
The above equation holds for all $j>k, \forall k \in [1,n]$. For the case $j=k, \forall k \in [1,n]$ we have $P_{k,k} \geq \frac{\rho^k}{k!}$. In addition for the base case ($k=1$) we have:
\begin{eqnarray}
P_{j,1}\geq P_{j-1,1} + \frac{\rho}{j} (1- {P_{j-1,1}}) , \forall j>1
\label{eq:conv2}
\end{eqnarray}
Furthermore, induction on $j$ and $k$ would show that  and $P_{j, k}$ is monotonically increasing with respect to both $j$ for a given  $k$ for $j \geq k$ and $k > 1$, of course with an upper bound of 1. In addition, for any finite iteration $P_{j-1,k}$ is smaller than $ {P_{j-1,k-1}}$ (one is subset of the other).  Our goal is to show that $P_{\infty,k}=P_{\infty,k-1}=P_{\infty,k-2}= ... = P_{\infty,1}=1$.

At first we will prove that $P_{\infty,1}=1$.  Let's focus on the inequality in Equation (\ref{eq:conv2}) and rewrite it to the following:
\begin{equation}
	P_{j, 1} - P_{j-1, 1} \geq (1 - P_{j-1, 1}) \cdot \frac{\rho}{j},  \forall j>1 
\end{equation}
We can then calculate the following summation
\begin{eqnarray}
\hspace{-0.75cm}	\sum_{j=2}^{\mu} \left( P_{j, 1} - P_{j-1, 1} \right) &\geq& 	\sum_{j=2}^{\mu} \left( (1 - P_{j-1, 1})  \cdot \frac{\rho}{j} \right) 	\nonumber \\
\hspace{-0.75cm}	P_{\mu, 1} - P_{1, 1} &\geq& 	\sum_{j=2}^{\mu} \left( (1 - P_{j-1, 1})  \cdot \frac{\rho}{j} \right) 	\nonumber
\end{eqnarray}
Taking $\mu$ to the limit at $\infty$ gives us
\begin{eqnarray}
\hspace{-1.25cm}	\lim_{\mu \rightarrow \infty} \left( P_{\mu, 1} - P_{1, 1}  \right) \geq \lim_{\mu \rightarrow \infty}  \sum_{j=2}^{\mu} \left( (1 - P_{j-1, 1}) \cdot \frac{\rho}{j} \right) 
\label{e:probInf}
\end{eqnarray}
Now, we can set  an upper bound $P_{j-1, 1} \leq \lambda_2$ where $\lambda_2 > 0$, and rewrite \eref{e:probInf} as
\begin{eqnarray}
	\lim_{\mu \rightarrow \infty} \left( P_{\mu, 1} - P_{1, 1}  \right) \geq \left( 1 - \lambda_2 \right) \lim_{\mu \rightarrow \infty}  \sum_{j=2}^{\mu} \frac{\rho}{j} 
	\label{eq:sumConv}
\end{eqnarray}
Since $\lim_{\mu \rightarrow \infty}  \sum_{j=k+1}^{\mu} \frac{\rho}{j} = \infty$, $\lambda_2 \geq 1$. However, since $\lambda_2$ is an upper bound of a probability value, $\lambda_2 \leq 1$. Therefore $\lambda_2$ must be 1, and is the least upper bound of $P_{j-1, 1}$ (Equation (\ref{eq:sumConv}) does not allow  $\lambda_2<1$). Since $P_{j-1, 1}$ is monotonically increasing in $j$, it will eventually approach to $\lambda_2=1$, thus $P_{\infty,1}=1$. In the rest of the proof we will show that $P_{\infty,k}=P_{\infty,k-1}, \forall k$ and thus  $P_{\infty,k}=1$ for all $k$.

  Let's now focus on the inequality in Equation (\ref{eq:conv1}) and rewrite it to the following:
\begin{equation}
	P_{j, k} - P_{j-1, k} \geq (P_{j-1,k-1} - P_{j-1, k}) \cdot \frac{\rho}{j},  \forall j>k
\end{equation}
We can then calculate the following summation
\begin{eqnarray}
\hspace{-0.75cm}	\sum_{j=k+1}^{\mu} \left( P_{j, k} - P_{j-1, k} \right) &\geq& 	\sum_{j=k+1}^{\mu} \left( (P_{j-1,k-1} - P_{j-1, k})  \cdot \frac{\rho}{j} \right) 	\nonumber \\
\hspace{-0.75cm}	P_{\mu, k} - P_{k, k} &\geq& 	\sum_{j=k+1}^{\mu} \left( (P_{j-1,k-1} - P_{j-1, k})  \cdot \frac{\rho}{j} \right) 	\nonumber
\end{eqnarray}
Taking $\mu$ to the limit at $\infty$ gives us
\begin{eqnarray}
\hspace{-1.25cm}	\lim_{\mu \rightarrow \infty} \left( P_{\mu, k} - P_{k, k}  \right) \geq \lim_{\mu \rightarrow \infty}  \sum_{j=k+1}^{\mu} \left( (P_{j-1,k-1} - P_{j-1, k}) \cdot \frac{\rho}{j} \right) 
\label{e:probInfV2}
\end{eqnarray}
Now, we can set  a lower  bound $(P_{j-1, k-1} -P_{j-1, k})  \geq \lambda_1^\prime$, and rewrite \eref{e:probInfV2} as
\begin{eqnarray}
	\lim_{\mu \rightarrow \infty} \left( P_{\mu, k} - P_{k, k}  \right) \geq \lambda_1^\prime \lim_{\mu \rightarrow \infty}  \sum_{j=k+1}^{\mu} \frac{\rho}{j} 
	\label{eq:sumConvV2}
\end{eqnarray}
Using similar arguments we used for the base case $(k=1)$, we get $\lambda_1^\prime =0$. Due to monotonicity properties and because  $X_{j-1,k}$ is a subset of $X_{j-1,k-1}$ (see Appendix A) thus $P_{j-1,k} < P_{j-1, k-1}$ for all finite $j>k$ then $P_{\infty,k}= P_{\infty,k-1}= ... =  P_{\infty,1}=1$. 
\end{proof}
 
Now, the question is how far away the cost of \ctrlSeq  is  from the optimal cost.


\begin{Lemma}
Let $\optCtsTraj: [0, T_g] \rightarrow \ctrlSpace$ be the optimal trajectory and \optCtrlSeq be the the sequence of control inputs that approximate \optCtsTraj with time intervals $\Delta t$, and $u:[0,T_g]\rightarrow U$ to be a trajectory that is $\epsilon$-close to \optCtsTraj, where $\Delta t$ is given by Equation (\ref{eq:deltaBound}). Suppose $\path{\phi^*(t)}$, $\path{u^*(t)}$ and $\path{u(t)}$ are the state space paths induced by \optCtsTraj, \optCtrlSeq(t), and $u(t)$ respectively. Then, $\| D(\path{\phi^*},t)-D(\path{u},t) \| \leq L_D E (e^{L_pt} -1) \leq L_D E (e^{L_pT_g} -1), \forall t \in [0,T_g]$, 
where $E=(\epsilon +m \alpha \Delta t)$ and $L_D$ is the maximum  Lipschitz constant for the cost function.
\label{l:costDist}
\end{Lemma}
\begin{proof} 
We consider 3 trajectories and the resulting 3 paths, the first trajectory is the optimal trajectory, the second trajectory is an approximation of the optimal trajectory and the last one is a trajectory that is $\epsilon$-close to the approximation of the optimal trajectory (that can be potentially  sampled by planners that explore the control space).

Using Lipschitz continuity on the cost function we get:
\begin{eqnarray}
\| D(\path{\phi^*},t)-D(\path{u},t) \| \leq  L_D(\| \path{\phi^*} - \path{u} \| )  
\label{eq:cost1}
\end{eqnarray}
 Where $L_D$ is the maximum  Lipschitz constant. From  Lemma (\ref{l:pathErr}) and  Lemma (\ref{l:disErr}) we have:
\begin{eqnarray}
\| \path{\phi^*} - \path{u} \| \leq E (e^{L_pt} -1) \nonumber \\ \leq E (e^{L_pT_g} -1),  \forall t \in [0,T_g]
\label{eq:cost2}
\end{eqnarray} 
Using Equation (\ref{eq:cost1}) and Equation (\ref{eq:cost2}) we get:
\begin{eqnarray}
\| D(\path{\phi^*},t)-D(\path{u},t) \| \nonumber \\ \leq L_D E (e^{L_pt} -1) \leq L_D E (e^{L_pT_g} -1), \forall t \in [0,T_g]
\end{eqnarray} 
\end{proof}

\begin{thm} The proposed family of algorithms, that sample directly the control space, return a trajectory that induces path with cost that asymptotically approaches to the optimal cost. 
Let $\phi^*: [0, T_g] \rightarrow \ctrlSpace$ be the optimal trajectory and \optCtrlSeq be  the sequence of control inputs that approximates $\phi^*$ with time intervals $\Delta t$, and $u^j:[0,T_g]\rightarrow U$ to be a trajectory that is $\epsilon$-close to  \optCtrlSeq returned by the algorithm until iteration $j$, where $\Delta t$ is given by Equation (\ref{eq:deltaBound}). Suppose $\gamma_{\phi^*}$, $\gamma_{u^*}$ and $\gamma_{u^j}$ are the state space paths induced by $\phi^*$, \optCtrlSeq, and $u^j$ respectively. If $D(\gamma_{u^j})$ is the cost of the path induced by trajectory $u^j$ after sampling $j$ samples in the control space then:
\begin{eqnarray}
\lim_{j \to \infty}[ {P(\lim_{\epsilon_d \to 0^+ }{\|D(\gamma^j_u) -D(\gamma_{\phi^*}) \|}  )} \leq \epsilon_d]=1 \nonumber
\end{eqnarray} 
Where $\epsilon_d \in \mathbb{R^+}$
\label{th:costDist}
\end{thm}
\begin{proof}
From Lemma (\ref{l:costDist}) and for any $\Delta t, \epsilon>0$ we get  $\epsilon_d \leq L_D E (e^{L_pT_g} -1)$, where $E=(\epsilon +m \alpha \Delta t)$, then there are choices of  $\epsilon = \epsilon^*$ and $\Delta t=\Delta t^*$ such that $\epsilon_d$ is arbitrarily small:
\begin{eqnarray}
\epsilon^*_d \leq L_D (\epsilon^* +m \alpha \Delta t^*) (e^{L_pT_g} -1)
\end{eqnarray} 
Where $\epsilon^*_d \in \mathbb{R^+}$ is arbitrarily  small, e.g. we can always find $\epsilon^*,\Delta t^*$ such that $\epsilon^*_d \leq \epsilon_d$.

From Theorem (\ref{th:prob}) we know that for any $\Delta t, \epsilon>0$ the probability of sampling a trajectory that is arbitrarily (for any $\Delta t, \epsilon$) close to the optimal trajectory approaches to 1 as the number of samples increases. 
\end{proof}

\subsection{The Lyapunov Approach}
Similar results to Theorem (\ref{th:prob}) can be obtained using the Lyapunov approach. We consider the system in Equation (\ref{eq:conv1}) and Equation (\ref{eq:conv2}), which is  in the discrete domain and describes the probability to get at-least one trajectory with $k$ milestones. We multiply both sides of  the above system times $-1$ and add $1$ in both sides and we get:
\begin{eqnarray}
1-P_{j,1}&\leq& 1-P_{j-1,1} - \frac{\rho}{j} (1- {P_{j-1,1}}) , \forall j>1 \\
1-P_{j,k}&\leq& 1-P_{j-1,k} - \frac{\rho}{j} (P_{j-1,k-1}- P_{j-1,k}) , \forall j>k , k>2
\label{eq:convCont1}
\end{eqnarray}
We define $Q_{j,k}=1-P_{j,k}, \forall k$. In terms of $Q_{j,k}$ the above System is written as follows:
\begin{eqnarray}
Q_{j,1}&\leq& Q_{j-1,1} - \frac{\rho}{j}Q_{j-1,1} , \forall j>1 \\
Q_{j,k}&\leq& Q_{j-1,k} - \frac{\rho}{j} (Q_{j-1,k} - Q_{j-1,k-1}) , \forall j>k , k>2
\label{eq:convCont2}
\end{eqnarray}

We transform this system to its equivalent in the continuous time domain and we get:
\begin{eqnarray}
\dot{Q}_{(t),1}&\leq& - \frac{\rho}{t}Q_{(t),1} , \forall t \geq t_1=\delta_t \nonumber \\
\dot{Q}_{(t),k}&\leq& - \frac{\rho}{t} (Q_{(t),k} - Q_{(t),k-1}) , \forall t \geq t_k = \delta_t k
\label{eq:convCont3}
\end{eqnarray}
Where $\delta_t \in \mathbb{R^+}$ is a time scaling constant that takes the iteration $i$ to time domain $t$, e.g. $t=\delta_t i$. The initial conditions of the above system are given as follows:
\begin{eqnarray}
{Q}_{(t=t_1),1}&=& (1-\rho) \nonumber \\ 
{Q}_{(t=t_2),2}&=& (1-\rho^2/2) \nonumber \\ 
&...&   \nonumber \\ 
{Q}_{(t=t_k),k}&=& (1-\frac{\rho^k}{k!}) \nonumber
\label{eq:convCont4}
\end{eqnarray}
Setting the ``velocities'' in the System in Equation (\ref{eq:convCont3}) equal to zero we can compute the equilibrium point ($Q_{eq,k}, \forall k$) as follows:
\begin{eqnarray}
Q_{eq,k}=Q_{eq,k-1}=...Q_{eq,1}=0 \nonumber
\end{eqnarray}

In order to show that  the System in  Equation (\ref{eq:convCont3})  reaches to the equilibrium point (e.g. $Q_{eq,k}=Q_{eq,k-1}=...Q_{eq,1}=0$) we will use the Lyapunov approach. We consider the following Lyapunov function: 
\begin{eqnarray}
 V(Q_{(t),1},Q_{(t),2}...,Q_{(t),k})= \sum_{i=1}^{k}{(Q_{(t),i})^2} 
\end{eqnarray}
Taking the time derivative of the above  Lyapunov function we get:
\begin{eqnarray}
 \dot{V}(Q_{(t),1},Q_{(t),2}...,Q_{(t),k})= 2\sum_{i=1}^{k}{(Q_{(t),i}) (\dot{Q}_{(t),i}) }
\end{eqnarray}
Using the System in  Equation (\ref{eq:convCont3}), we get
\begin{eqnarray}
 \dot{V}(Q_{(t),1},Q_{(t),2}...,Q_{(t),k}) \leq  -\frac{2 \rho}{t} \left((Q_{(t),1})^2 +\sum_{i=2}^{k}{(Q_{(t),i}) ( {Q}_{(t),i} - {Q}_{(t),i-1}  ) } \right)
\end{eqnarray}

It is important to see that the Lyapunov function is equal to zero at the equilibrium point and its derivative is always negative (but at the equilibrium point). Therefore, the system in  Equation (\ref{eq:convCont3}) reaches  its equilibrium value (e.g. $Q_{eq,k}=Q_{eq,k-1}=...Q_{eq,1}=0$).

\subsection{Convergence Rate}
In this subsection we will perform an analysis on the convergence rate for algorithms that sample the control space directly (using uniform sampling and without pruning). In the previous subsection we have shown that the system in  Equation (\ref{eq:convCont3}) approaches its Equilibrium point (e.g. $Q_{eq,k}=Q_{eq,k-1}=...Q_{eq,1}=0$) and therefore $P_{\infty,k}=1, \forall k$. The question arises here is how fast the system in question approaches to its equilibrium point.

We re-write the system in Equation (\ref{eq:convCont3})  in the following form:
\begin{eqnarray}
\dot{Q}_{(t),k} + \frac{\rho}{t} Q_{(t),k} &\leq& \frac{\rho}{t} Q_{(t),k-1}, \forall t > t_k = \delta_t k \nonumber \\ 
\dot{Q}_{(t),k-1} + \frac{\rho}{t} Q_{(t),k-1} &\leq& \frac{\rho}{t}  Q_{(t),k-2}, \forall t > t_{k-1} = \delta_t (k-1) \nonumber \\ 
&...& \nonumber \\
\dot{Q}_{(t),2} + \frac{\rho}{t} Q_{(t),2} &\leq& \frac{\rho}{t}  Q_{(t),1}, \forall t > t_{2} = \delta_t (2) \nonumber  \\
\dot{Q}_{(t),1} + \frac{\rho}{t} Q_{(t),1} &\leq& 0, \forall t > t_{1} = \delta_t  \nonumber 
\label{eq:convConvergence1}
\end{eqnarray}
With initial conditions as described in the previous section.

At first we will solve for $Q_{(t),1}$. Separating variables ($t$ and $Q_{(t),1}$) and integrating both parts we get:
\begin{eqnarray}
\ln{(Q_{(t),1})} \leq \ln{(t^{-\rho})} + \ln{(C)}
\end{eqnarray}
Where $C$ is the integration constant. To compute $C$ we use the initial conditions e.g. for $t=t_1, Q_{(t),1}=(1-\rho)$.  Using the initial condition and after some algebraic manipulation  we get: 
\begin{eqnarray}
{Q_{(t),1}} \leq (t^{-\rho}) (1-\rho) t_1^\rho
\end{eqnarray}
Similar results we get for the discrete system, see Appendix (B).

Now we will compute $Q_{(t),2}$. To do so we multiply both sides of the the differential equation in question with $\mu(t)=e^{\int \frac{\rho dt}{t}}=t^\rho$ and we get:
\begin{eqnarray}
\dot{Q}_{(t),2}t^\rho + \frac{\rho}{t} Q_{(t),2} t^\rho \leq t^\rho Q_{(t),1}
\end{eqnarray}
For the the left hand side of the above equation we get:  $\dot{Q}_{(t),2}t^\rho + \frac{\rho}{t} Q_{(t),2} t^\rho=\frac{d((t^\rho)Q_{(t),2})}{dt}$
Using results from the previous milestone we get:
\begin{eqnarray}
\frac{d((t^\rho)Q_{(t),2})}{dt} \leq \frac{\rho}{t} (1-\rho) t_1^\rho
\end{eqnarray}
Separating  variables and using the initial conditions (e.g. $Q_{(t_2),2}=(1-\frac{\rho^2}{2!})$) we get:
\begin{eqnarray}
Q_{(t),2} \leq (t^{-\rho}) [\alpha_1 \ln(t) + \alpha_2 ]
\end{eqnarray}
Where $\alpha_1=\rho (1-\rho) t_1^\rho$, $\alpha_2= (1-\frac{\rho^2}{2!}) t_2^\rho - \alpha_1 \ln(t_2)$.

Using similar techniques we can show that:
\begin{eqnarray}
Q_{(t),3} \leq (t^{-\rho}) [ b_1 (\ln(t))^2 + b_2 \ln(t) +b_3 ]
\end{eqnarray}
Where $b_1 = \frac{\alpha_1 \rho}{2}$, $b_2 = \alpha_2 \rho$, $b_3= (1- \frac{\rho^3}{3!}) t_3^\rho - b_1 (\ln(t_3))^2 - b_2 \ln(t_3)$.

 For $Q_{(t),4}$ we get:
\begin{eqnarray}
Q_{(t),4} \leq (t^{-\rho}) [  c_1 (\ln(t))^3 +   c_2 (\ln(t))^2 + c_3 \ln(t) +c_4 ]
\end{eqnarray}
Where $c_1= \frac{\rho b_1}{3}$, $ c_2=\frac{\rho b_2}{2}$, $c_3= \rho b_3$, $c_4= (1- \frac{\rho^4}{4!})t_4^\rho - c_1 (\ln(t_4))^3 - c_2 (\ln t_4)^2 - c_3 \ln(t_4)$.
 
 Now we can clearly see a pattern for the solution to the system in   Equation (\ref{eq:convCont3}). For the general case we will use induction. We assume that:
\begin{eqnarray}
Q_{(t),n-1} \leq (t^{-\rho}) [  d_1 (\ln(t))^{n-2} +   d_2 (\ln(t))^{n-3}... d_{n-2} \ln(t) +d_{n-1} ]
\end{eqnarray} 
Where $d_i, \forall i=[1,n-1]$ are constants.
 
 Using the $n$--th equation from the system  in   Equation (\ref{eq:convCont3}) we have:
\begin{eqnarray}
  \dot{Q}_{(t),n} + \frac{\rho}{t} Q_{(t),n}  \leq  \frac{\rho}{t}  (t^{-\rho}) [  d_1 (\ln(t))^{n-2} +   d_2 (\ln(t))^{n-3}... d_{n-2} \ln(t) +d_{n-1} ]
\end{eqnarray}  
 Multiplying both sides with $t^\rho$ we get:
\begin{eqnarray}
  \dot{Q}_{(t),n}t^\rho + \frac{\rho}{t} Q_{(t),n}t^\rho  &\leq&  \frac{\rho}{t} t^\rho (t^{-\rho}) [  d_1 (\ln(t))^{n-2} +   d_2 (\ln(t))^{n-3}... d_{n-2} \ln(t) +d_{n-1} ] \nonumber \\
  &=& \frac{\rho}{t} [  d_1 (\ln(t))^{n-2} +   d_2 (\ln(t))^{n-3}... d_{n-2} \ln(t) +d_{n-1} ]
\end{eqnarray}
Integrating both parts we get
\begin{eqnarray}
 Q_{(t),n} t^\rho \leq \int \frac{\rho}{t} [  d_1 (\ln(t))^{n-2} +   d_2 (\ln(t))^{n-3}... d_{n-2} \ln(t) +d_{n-1} ] dt +C
\end{eqnarray}
Where $C$ is an integration constant. Using integration tables we know that:
\begin{eqnarray}
\int \frac{(\ln(t))^n dt}{t}= \frac{(\ln(t))^{n+1}}{n+1}
\end{eqnarray}
Using the above mathematical  formula, we get:
\begin{eqnarray}
 Q_{(t),n} \leq  t^{-\rho} [ e_1 (\ln(t))^{n-1}  +  e_2  (\ln(t))^{n-2}   + e_3  (\ln(t))^{n-3}  ... e_{n-1} (\ln(t)) + e_n ]
 \label{conveFinal}
 \end{eqnarray}
Where the constants $e_i = \frac{\rho d_i}{n-i}, \forall i \in [1,n-1]$ and $e_n$ is the constant $C$ computed using the initial conditions (e.g. $Q(t_n)= (1- \frac{\rho^n}{n!})$).

It is important to notice that leading term in Equation (\ref{conveFinal}) is given by:  $ e_1 (\ln(t))^{n-1}t^{-\rho}$, where $e_1=t_1^\rho \frac{(\rho^{n-1})(1-\rho)}{(n-1)!}$. Using the  ``L'Hopital's Rule'' $n-1$ times \cite{LHopital} we see that it goes to zero (as expected). Another form of  Equation (\ref{conveFinal}) is as follows:
\begin{eqnarray}
 Q_{(t),n} \leq  t^{-\rho} [ e'_1 \frac{(\ln(t))^{n-1}}{(n-1)!}  +  e'_2  \frac{(\ln(t))^{n-2}}{(n-2)!}   + e'_3  \frac{(\ln(t))^{n-3}}{(n-3)!}  ... e'_{n-1} \frac{(\ln(t))}{1} + e_n ]
 \label{conveFinal2}
 \end{eqnarray}
 Where $ e'_i$ are constants.

The above inequality gives a bound on the probability ``not to get at least one trajectory with $n$ milestones that is $\epsilon$--close to an approximation of the optimal trajectory (with $n$ milestones and sufficiently small $\Delta t$ such that $n = \lfloor \frac{T_g}{\Delta t} \rfloor$ )''  for any positive $\rho$, which  is proportional to $\epsilon$. In the case, that the integration time is sampled uniformly form an interval e.g. $[0,\tau], \tau>0$, then the probability not  to get at least one  trajectory that is $\epsilon$--close to \emph{any} approximation of the optimal trajectory, with $n$ milestones (each of them resulted in after applying control input for at most $\Delta t$), is given by a similar inequality (change $\rho$ to $\rho \frac{\Delta t}{\tau}$).  However to cover an optimal trajectory with time duration $T_g$ it takes $\frac{2 T_g}{\Delta t}$ expected number of milestones.

\subsection{Different Sampling Strategies}

Both the optimality proof and the convergence analysis hold for planners that sample directly the control space by choosing a node to expand uniformly at random, applying control input at random and without using pruning.  One can use artificial intelligence methods to guide sampling and thus improve the convergence  rates. For example, let the probability to choose the ``correct'' node for expansion to be $f_1(\frac{1}{j})>\frac{1}{j}$ (where $f_1:\mathbb{R^+} \rightarrow \mathbb{R^+}$), then the summation in the right-hand side of Equation (\ref{eq:sumConv}) approaches faster to infinity and thus the probability $P_{j,k}$ approaches faster to 1.   One way to do this is to use pruning techniques, however when we use pruning techniques it is very difficult to compute the convergence rate.

In addition, we see in the previous section the convergence rate depends directly on  $\rho$. We recall that $\rho$ is the probability to choose a control input that is within a ball of radius $\epsilon$ centered at the optimal one. In the case we use uniform sampling of the control space  $\rho$ is given by the ratio of the volume of that ball divided by the volume of the control space. Therefore, if we use artificial intelligence methods to guide the control sampling process then we can directly increase $\rho$ and therefore effect the convergence of the proposed approach. Studying and analyzing different sampling strategies is not within the scope of this research; we recall that  the goal of this work is to show that sampling in the control space directly  can achieve asymptotically optimal planning.

\section{Simulation Results}
\label{s:results}
This section presents performance comparison between the simplest sampling strategy we have analyzed in the previous section with more sophisticated strategies, i.e., uniform with pruning and expandTreeRRT with pruning, which is a simplified version of Sparse-RRT\cite{Littlefield2013IROS}. For this comparison, we use the problem of ``Pendulum on a Cart'' with obstacles (see Figure (\ref{fig:CartAndPendulum})).
 \begin{figure}[!htpb]
\begin{center}
\includegraphics[height=2.5in,width=3in] {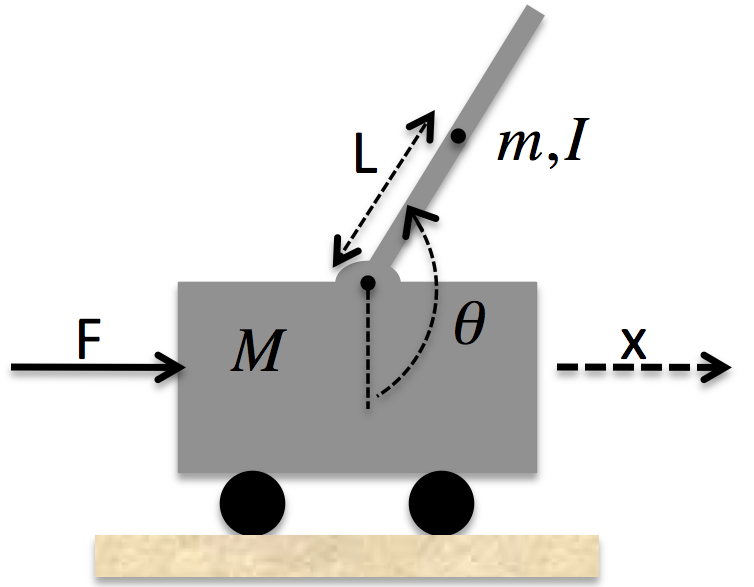} 
\end{center}
\caption{Pendulum on a Cart}
\label{fig:CartAndPendulum}
\end{figure}


This is a highly-non-linear 2$^{nd}$ order under-actuated  system with infinite number of  equilibrium points both stable and unstable. The system's state space is 4-dimensional, e.g. $X \in \mathbb{R}^4$, with $X=[x,v,\theta,\Omega]^T$, where $x$ indicates the displacement of the cart ( with mass $M$), $v$ is the velocity of the cart (e.g. $v=\dot{x}$), $\theta$ indicates the angle of the pendulum (with mass $m$, inertia $I$ and length $L$), and $\Omega$ is the angular velocity of the pendulum (e.g. $\Omega=\dot{\theta}$ ). Using the  Euler--Lagrange equations we can derive the equations of motion:
\footnotesize{
\begin{eqnarray}
\dot{x}(t)&=&v(t) \nonumber \\ 
\dot{v}(t)&=&  \frac{(I+mL^2)(F(t)+ mL\Omega^2(t) \sin(\theta(t)))  + (mL)^2 \cos(\theta(t))\sin(\theta(t))g}{ (M+m)(I+mL^2) - (mL)^2 \cos^2(\theta(t))} \nonumber  \\
\dot{\theta}(t)&=&\Omega(t) \nonumber \\
\dot{\Omega}(t)&=& \frac{(-mL\cos(\theta(t)))(F+mL\Omega^2(t) \sin(\theta(t))) +(M+m)(-mgL\sin(\theta(t)))  }{(M+m)(I+mL^2) - (mL)^2 \cos^2(\theta(t))} \nonumber \\
\end{eqnarray}
}
\normalsize
For this problem we would like to minimize control effort and time to the goal, thus we define as a cost function the following:
\begin{eqnarray}
D(F)= \int_0^{T_f} (a_fF^2 +a_t  )\,dt.
\end{eqnarray}
where: $a_t=1000 \frac{cost}{s}$,  $a_f=1 \frac{cost}{N^2 s}$, $ M=10kg, m=5kg, I=10kgm^2, L=2.5 m$ (see Figure (\ref{fig:CartAndPendulum})), $g=9.86\frac{N}{s^2}$, the input to the system is a force $F$ acting on the cart in the $x$ direction, for this case $F$ is uniformly sampled e.g.  $F=[0,300]N$. The initial conditions are all zero,  the workspace is such that: $x=[0,60]m$, and includes obstacles  as shown in Figure (\ref{fig:pendResults}).  The goal region is located in the right side $(48m< x <52 m)$ in the upright position ($(180-10)^{\circ}<\theta<(180+10) ^{\circ}$) with $  -3.14 rad/s <\Omega <3.14 rad/s$ and  $(-4m/s< v < 4m/s)$.\footnote{the goal here is not to stabilize the system in the upright position. After a method like the one described here brings the system close to the upright position, one can use closed loop control methods to stabilize it there.}

\begin{figure}[!th]
\begin{center}
\fbox{\includegraphics[ height=1.2in]{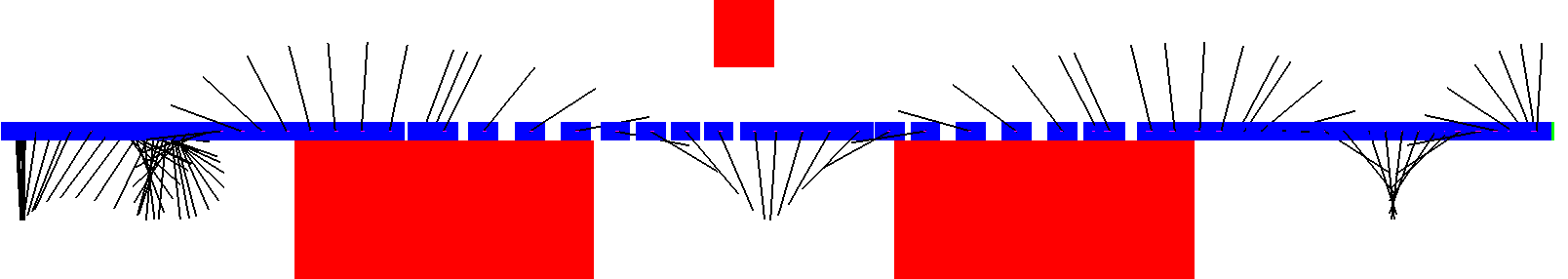}}
\caption{Pendulum on a Cart:  A typical close-to-optimal trajectory.}
\label{fig:pendResults}
\end{center}
\end{figure}

We implemented the algorithms in C++ in a Dell computer that runs Linux on  32  Intel $x86-64$ processors at  1.2 GHz.  Each one of the results presented here represents average results over 21 runs.
Figure (\ref{f:cartPole}) shows the cost of the trajectories generated by  simple-uniform (the algorithm we used for analysis), simple-uniform with pruning, and expandTreeRRT with pruning.  As expected, the expandTreeRRT method computes an admissible trajectory faster than the uniform approaches.  What is interesting is the expandTreeRRT with pruning converges to a further lower cost ($2\times10^8$) much slower than simple-uniform with and without pruning. This happens because the optimal solution occupies only a small region of the state space. To reach a near optimal solution fast, one needs to sample more densely near the region of optimal solution. RRT's expansion strategy which tries to cover the entire state space well will actually ``under sample" state space region that are near optimal solution and ``over sample" other regions that do not contribute in generating the optimal trajectory. Note that this does not mean that the expansion strategy in simple-uniform is a suitable one, rather the simulation results indicate that finding sampling strategy that biases sampling towards regions near optimal trajectory would be a more fruitful avenue. 

\begin{figure}[!h]
	\centering
		\includegraphics[width=3in] {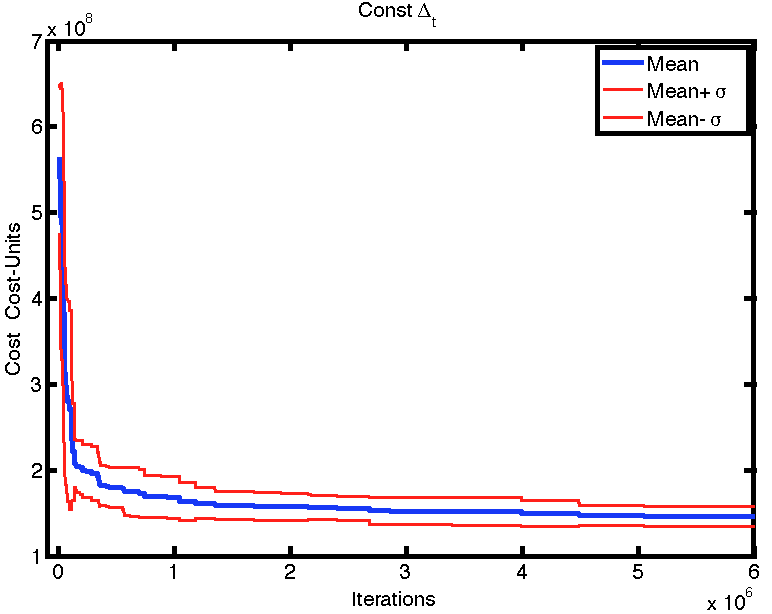}  \\
		\includegraphics[width=3.5in] {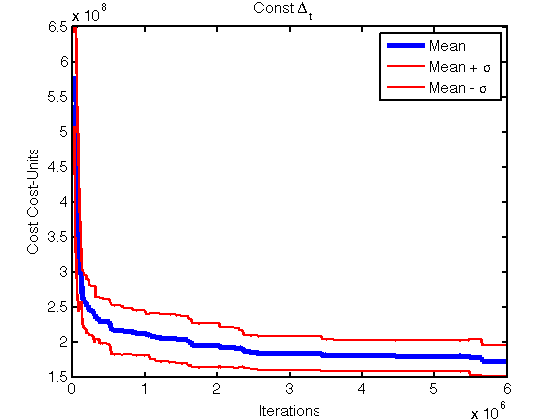}  \\
		\includegraphics[width=3.5in] {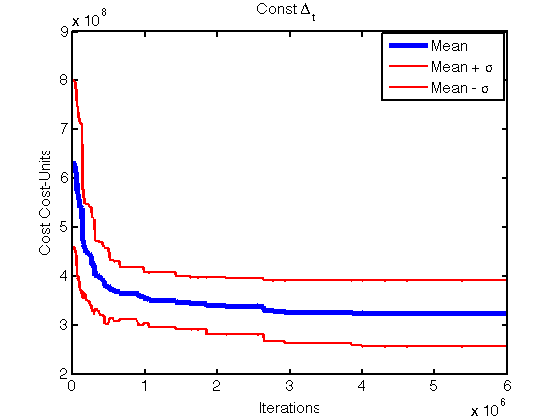} 
	\caption{Cart and pole simulation results. Upper: Simple-uniform without pruning. Middle: Simple-uniform with pruning. Bottom: ExpandTreeRRT with pruning. The iterations shown are the valid iterations (do not count collisions), the actual iterations are approximate 2 times more.}
	\label{f:cartPole}
\end{figure}

 \begin{figure}[!htpb]
\begin{center}$
\begin{array}{cc}
\includegraphics[height=2.5in,width=3in] {FiguresChapter3/PendulumConstantTime.png} & \includegraphics[height=2.5in, width=3in] {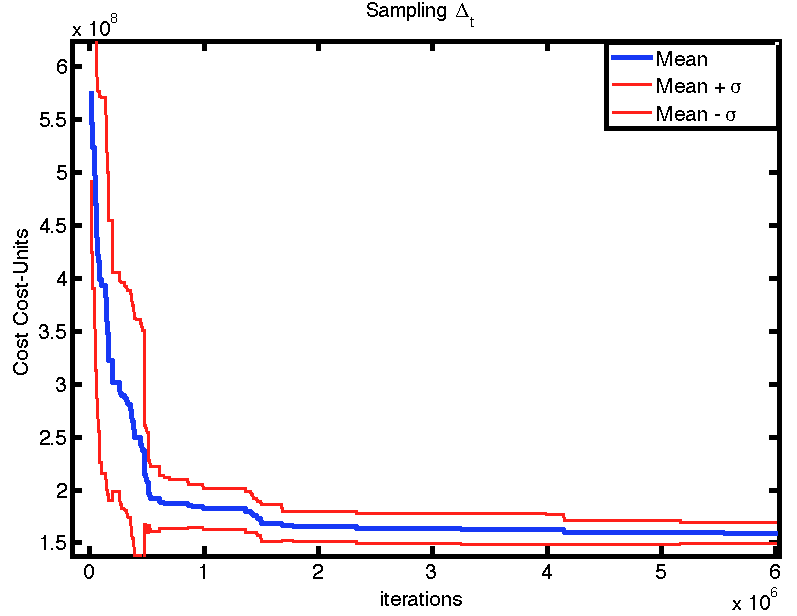} 
\end{array}$
\end{center}
\caption{Pendulum on a Cart, statistics (results obtained by choosing a node to expand uniformly and using pruning).  Left hand side: the statistics for the case  constant integration time $\Delta t=1s$ is used. Right hand side:  the case the integration time is uniformly sampled from $[0,3s]$ is used. The iterations shown are the valid iterations (do not count collisions), the actual iterations are approximate 2 times more.}
\label{fig:pendStatistics}
\end{figure}
In addition to comparing sampling strategies, we also try to understand the effect of different ways of setting time discretization. Figure (\ref{fig:pendStatistics}) shows the statistics for the case of constant integration time $\Delta t=1s$ and in the case we sample integration time from $[0,3s]$. In both cases the results were taken using simple-uniform with pruning. When we use constant integration time (which is sufficiently small) we get results similar to the ones we get when we sample the integration time, however the results for constant integration time, for small number of iterations are slightly better than the ones we get by sampling the integration time.

\section{Summary and Discussion}

Most methods for solving  the problem of optimal motion planning use direct exploration of the configuration space. These methods sample the configuration space and rely on steering functions to connect the configuration space with the control space. However, for general non-linear systems, a steering function is not always available and can be expensive to compute. This paper shows that we can solve optimal motion planning problem without steering functions by sampling the control space directly.

In this paper, we present a novel analysis on asymptotic optimality of a family of algorithms that directly sample the control space. Our analysis is based on the simplest method in this family of algorithms. We also present a  comparison result between the simplified method we used for analysis and more sophisticated sampling methods in this family of algorithms.  As the number of iterations increases, the trajectory these algorithms sample approaches the optimal trajectory. 

Many avenues are possible for future improvement. Can we generalize the theoretical results further. For instance, will the asymptotic optimality property holds when Lipschitz continuity does not hold. Furthermore, this paper only shows that it is possible to solve optimal motion planning problems without a steering function. How to design such an efficient optimal motion planner  remains an open problem.

\subsubsection*{Acknowledgments}
This work was supported by the Singapore-MIT Alliance for Research and Technology (SMART) Center for Environmental Sensing and Modeling (CENSAM). We would also like to thank Professor Sertac Karaman for useful discussions.
\begin{appendices}

\section*{Appendix A: Probability to Get at Least one Trajectory that is $\epsilon$--close to Any Approximation of the Optimal Trajectory}
\label{appA}

Let the event $X_{i,k}$ be the following event: `` From $j$ valid samples (up to $j$ iterations),  an algorithm, that samples the control space, generates at least one control sequence \ctrlSeq that is $\epsilon$-close to the first $k^{th}$ subsequence of \optCtrlSeq, i.e, $|\ctrlSeq| = k$ and $dist(u_i, u^*_i) \leq \epsilon$ for $i \in [0, k]$.  Then using the total probability theorem and condition on the events $X_{i-1,k}$ and  $\bar{X}_{i-1,k}$ (up to the previous iteration the underlining algorithm  (does not) return(s) at least one trajectory with the desired characteristics) is given by:
\begin{eqnarray}
Pr(X_{j,k})&=& Pr(X_{j,k} | X_{j-1,k}) Pr(X_{j-1,k}) +  Pr(X_{j,k} | \bar{X}_{j-1,k}) (1-Pr(X_{j-1,k}) ) \Rightarrow \nonumber \\
Pr(X_{j,k})&=& Pr(X_{j-1,k}) +  Pr(X_{j,k} | \bar{X}_{j-1,k}) (1-Pr(X_{j-1,k}) )
\end{eqnarray}
To compute the term  $Pr(X_{j,k} | \bar{X}_{j-1,k}) $ in the above equation, we use the total probability theorem condition on the events $X_{j-1,k-1}$ and  $\bar{X}_{j-1,k-1}$. 
\begin{eqnarray}
Pr(X_{j,k} | \bar{X}_{j-1,k})= &Pr&(X_{j,k} | \bar{X}_{j-1,k} \bigcap \bar{X}_{j-1,k-1}) Pr(\bar{X}_{j-1,k-1} | \bar{X}_{j-1,k}) +  \nonumber \\ 
&Pr&(X_{j,k} | \bar{X}_{j-1,k} \bigcap {X}_{j-1,k-1}) Pr({X}_{j-1,k-1} | \bar{X}_{j-1,k})  
\end{eqnarray}
Because $X_{j-1,k}$ is a subset of  $X_{j-1,k-1}$, then $ Pr(X_{j,k} | \bar{X}_{j-1,k} \bigcap \bar{X}_{j-1,k-1}) = Pr(X_{j,k} | \bar{X}_{j-1,k-1})=0$. In addition, by the definition of $\rho$ for the second term in the above equation we get $Pr(X_{j,k} | \bar{X}_{j-1,k} \bigcap {X}_{j-1,k-1}) \geq  \frac{\rho}{j}$. Where $\frac{1}{j}$ is the probability to sample any node for expansion (in the case we use uniform sampling). Similar analysis holds for the case we do not use uniform sampling. Using the above we get:
\begin{eqnarray}
Pr(X_{j,k} | \bar{X}_{j-1,k}) &\geq& \frac{\rho}{j}Pr({X}_{j-1,k-1} | \bar{X}_{j-1,k}) \nonumber \\
Pr(X_{j,k})&\geq& Pr(X_{j-1,k}) + \frac{\rho}{j}Pr({X}_{j-1,k-1} | \bar{X}_{j-1,k})  (1-Pr(X_{j-1,k}) )
\label{eqAP:1}
\end{eqnarray}
 Using the Bayes' theorem (alternatively Bayes' law or Bayes' rule), we get
\begin{eqnarray}
Pr({X}_{j-1,k-1} | \bar{X}_{j-1,k})= \frac{Pr( \bar{X}_{j-1,k}| {X}_{j-1,k-1}) Pr({X}_{j-1,k-1})} {Pr( \bar{X}_{j-1,k})} 
\label{eqAP:2}
\end{eqnarray} 
Using Equation (\ref{eqAP:1}) and Equation (\ref{eqAP:2}) we get:
\begin{eqnarray}
Pr(X_{j,k})\geq Pr(X_{j-1,k}) + \frac{\rho}{j} Pr( \bar{X}_{j-1,k}| {X}_{j-1,k-1}) Pr({X}_{j-1,k-1}) \Rightarrow \nonumber \\
Pr(X_{j,k})\geq Pr(X_{j-1,k}) + \frac{\rho}{j} (1- Pr({X}_{j-1,k}| {X}_{j-1,k-1}) ) Pr({X}_{j-1,k-1})
\label{eqAP:3}
\end{eqnarray}

Using,  the definition of conditional probabilities and because ${X}_{j-1,k}$ is a subset of   ${X}_{j-1,k-1}$
\begin{eqnarray}
Pr({X}_{j-1,k}| {X}_{j-1,k-1}) = \frac{Pr({X}_{j-1,k} \bigcap {X}_{j-1,k-1})}{Pr({X}_{j-1,k-1})}= \frac{Pr({X}_{j-1,k})}{Pr({X}_{j-1,k-1})}
\label{eqAP:4}
\end{eqnarray}
Substituting Equation (\ref{eqAP:4}) in  Equation (\ref{eqAP:3}) we get:
\begin{eqnarray}
Pr(X_{j,k})&\geq& Pr(X_{j-1,k}) + \frac{\rho}{j} (1- \frac{Pr({X}_{j-1,k})}{Pr({X}_{j-1,k-1})}) Pr({X}_{j-1,k-1}) \\
Pr(X_{j,k})&\geq& Pr(X_{j-1,k}) + \frac{\rho}{j} (Pr({X}_{j-1,k-1})- Pr({X}_{j-1,k}))\label{eqAP:5}
\end{eqnarray}

The above Equation holds for all $j>k, \forall k \in [1,n]$. For the case $j=k, \forall k \in [1,n]$ we have $P_{k,k} \geq \frac{\rho^k}{k!}$. In addition for the base case we have:
\begin{eqnarray}
Pr(X_{j,1})\geq Pr(X_{j-1,1}) + \frac{\rho}{j} (1- \frac{Pr({X}_{j-1,1})}{Pr({X}_{j-1,0})}) Pr({X}_{j-1,0}) \nonumber \Rightarrow \\ 
Pr(X_{j,1})\geq Pr(X_{j-1,1}) + \frac{\rho}{j} (1- {Pr({X}_{j-1,k})}) 
\label{eqAP:6}
\end{eqnarray}

For simplicity, in the main part of this paper we use $P_{i,k}$ instead of $Pr(X_{j,k})$.

\section*{Appendix B: Convergence Rate for the First Milestone for the Discrete System}
\label{appB}

Here we will perform an analysis on the discrete system for the convergence rate for the first milestone (e.g. $k=1$). Again, we consider uniform sampling without pruning. To simplify our notations, Instead of using inequalities we will use equalities and work with bounds. We consider another system as follows:
\begin{eqnarray}
\hat{P}_{j,1}&=&\hat{P}_{j-1,1} +(1-\hat{P}_{j-1,1} )\frac{\rho}{j}  , \forall j \geq k, k=1
\end{eqnarray}
Thus, $\hat{P}_{j,1} \leq P_{j,1}$, we define $Q_{j}=1-\hat{P}_{j,1}$ thus
\begin{eqnarray}
Q_j&=&Q_{j-1}(1-\frac{\rho}{j}) , \forall j \geq k, k=1
\label{e:inverseSystem}
\end{eqnarray}
It is easy to see that as the number of iterations increases $Q_j$ goes to 0. Let $\alpha_1$ be the convergence rate then we have:
\begin{eqnarray}
Q_j&=& j^{-\alpha_1} \\
Q_{j-1}&=& {(j-1)}^{-\alpha_1}
\end{eqnarray}
Substituting this equation on the above equations we get:
\begin{eqnarray}
(\frac{j}{j-1})^{-\alpha_1} &=& (1- \frac{\rho}{j}) \nonumber \\
\log{ (\frac{j}{j-1})^{-\alpha_1}} &=& \log{(1- \frac{\rho}{j})}  \nonumber \\
\log{ (\frac{j-1}{j})^{\alpha_1}} &=& \log{(1- \frac{\rho}{j})}  \nonumber \\
\alpha_1&=&\frac{  \log{(1- \frac{\rho}{j})} }{ \log{ (1-\frac{1}{j})  }}
\end{eqnarray}

Using Taylor series \cite{lagrange1813theorie} we have $\log(1-x)\approx -x, x <1 \in \mathbb{R}$  thus the convergence rate for large $j$ is equal to $\rho$. More formally,
\begin{eqnarray}
\lim_{j \to \infty}{\frac{  \log{(1- \frac{\rho}{j})} }{ \log{ (1-\frac{1}{j})  }}} = \frac{0}{0}
\end{eqnarray}
Using the  ``L'Hopital's Rule'' \cite{LHopital} we get: 
\begin{eqnarray}
\lim_{j \to \infty}{\frac{  \log{(1- \frac{\rho}{j})} }{ \log{ (1-\frac{1}{j})  }}} = \lim_{j \to \infty} { \frac{ \frac{\rho}{j^2 \|1-  \rho/j \|}   }{  \frac{1}{j^2 \|1-  1/j \|}  }}= \rho
\end{eqnarray}

The above results  agree with our analysis for the continuous equivalent system.


\end{appendices}

\bibliographystyle{apalike}
\bibliography{PathPlanningLibrary.bib,ROB-13-0005xx_myLibrary}
\end{document}